\documentclass[11pt,a4paper]{article}
\usepackage{graphicx}
\usepackage[T1]{fontenc}
\usepackage[french,polutonikogreek,polish,english]{babel}
\usepackage{amsfonts,amsmath,amsthm}
\usepackage{fullpage}
\usepackage{subcaption}
\usepackage{makecell}
\newtheorem{theorem}{Theorem}
\newtheorem{definition}[theorem]{Definition}
\newtheorem{problem}[theorem]{Problem}
\newtheorem{proposition}[theorem]{Proposition}
\newtheorem{lemma}[theorem]{Lemma}
\newtheorem{corollary}[theorem]{Corollary}
\newtheorem{remark}{Remark}
\newtheorem{note}{Note}

\begin{document}

\date{}

\title{Computing the $k$-resilience of a Synchronized Multi-Robot System\thanks{A 4-page abstract of this paper appeared in the informal and non-selective workshop EuroCG17 \cite{eurocg17}. This research has received funding from the projects GALGO (Spanish Ministry of Economy and Competitiveness, MTM2016-76272-R AEI/FEDER,UE) and CONNECT (EU-H2020/MSCA under grant agreement 2016-734922).}
}

\author{Sergey Bereg\thanks{
Department of Computer Science,
University of Texas at Dallas,
800 West Campbell Road.
Richardson, TX 75080.
USA}
\and
        Luis-Evaristo Caraballo\thanks{Higher Technical School of Engineering,
              University of Seville,
              Camino de los Descubrimientos. 
              E-41092 Seville.
              Spain}
\and 
        Jos\'e-Miguel D\'iaz-B\'a\~nez$^\ddag$
\and
        Mario A. Lopez\thanks{Department of Computer Science,
	          University of Denver,
	          2360 South Gaylord Street. 
	          Denver, CO 80208.
	          USA}
}

\maketitle

\begin{abstract}
We study an optimization problem that arises in the design of covering strategies for multi-robot systems. Consider a team of $n$ cooperating robots traveling along predetermined closed and disjoint trajectories. Each robot needs to periodically communicate information to nearby robots. At places where two trajectories are within range of each other, a communication link is established, allowing two robots to exchange information, provided they are ``synchronized'', i.e., they visit the link at the same time. In this setting a communication graph is defined and a system of robots is called \emph{synchronized} if every pair of neighbors is synchronized. 

If one or more robots leave the system, then some trajectories are left unattended. To handle such cases in a synchronized system, when a live robot arrives to a communication link and detects the absence of the neighbor, it shifts to the neighboring trajectory to assume the unattended task. 
If enough robots leave, it may occur that a live robot enters a state of \emph{starvation}, failing to permanently meet other robots during flight. 
To measure the tolerance of the system under this phenomenon we define the \emph{$k$-resilience} as the minimum number of robots whose removal may cause $k$ surviving robots to enter a state of starvation.
We show that the problem of computing the $k$-resilience is NP-hard if $k$ is part of the input, even if the communication graph is a tree. We propose algorithms to compute the $k$-resilience for constant values of $k$ in general communication graphs and show more efficient algorithms for systems whose communication graph is a tree.
\end{abstract}

\section{Introduction}

Recently, there has been a growing interest in systems composed of multiple autonomous mobile robots that exhibit some kind of cooperative behavior. Many interesting algorithmic and combinatorial problems arise naturally in the design of coordinated multi-robot systems \cite{acevedo_jint14,alejo2013velocity,chevaleyre2004theoretical,cusick1973view,kranakis2011boundary,dumitrescu2014fence,kawamura2015fence,kawamura2015simple}.

Scalability, fault-tolerance, and failure-recovery are important concerns of distributed systems. 
In recent papers, \cite{dbanez2015icra,Tro-paper} consider a scenario consisting of $n$ robots each of which periodically travels along a predetermined closed trajectory (the trajectories are pairwise disjoint) while performing an assigned task. 
Each robot needs to communicate information about its operation to other robots, but the communication interfaces have a limited range. 
Hence, when two robots are within range, a communication link is established, and information is exchanged. Accordingly, the set of potential communication links determines a graph with trajectories as nodes and links as edges. Two trajectories are \emph{neighboring} if they are adjacent in the graph of potential links. Two robots are \emph{neighbors} if they occupy neighboring trajectories.
Two neighboring robots are \emph{synchronized} if they can exchange information by periodically being within communication range of each other. 
Given the robot trajectories in the plane and the communication range of the robots, the \emph{synchronization problem} is to schedule (if possible) the movement of robots along trajectories so that every pair of neighboring robots is synchronized, in which case it is said that the system is \emph{synchronized}.

\begin{figure}
	\begin{subfigure}{.245\textwidth}
		\centering
		\includegraphics[page=5]{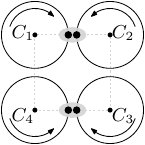}
		\caption{}
		\label{fig:synchro_sample_a}
	\end{subfigure}
	\begin{subfigure}{.245\textwidth}
		\centering
		\includegraphics[page=6]{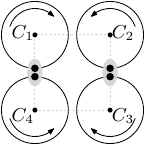}
		\caption{}
		\label{fig:synchro_sample_b}
	\end{subfigure}
	\begin{subfigure}{.245\textwidth}
		\centering
		\includegraphics[page=7]{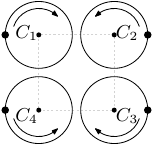}
		\caption{}
	\end{subfigure}
	\begin{subfigure}{.245\textwidth}
		\centering
		\includegraphics[page=8]{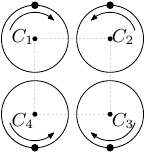}
		\caption{}
	\end{subfigure}
	\caption{Synchronized system of robots following circular trajectories. The robots are represented as solid points in the circles. Arrows represent the movement direction of the robots in the trajectories.	In this example the graph of potential links is a cycle of four nodes.}
	\label{fig:synchro_sample}
\end{figure}

Figure~\ref{fig:synchro_sample} shows a synchronized system where every pair of neighboring robots are moving in opposite directions (one clockwise and the other counterclockwise) at the same constant speed along congruent circular trajectories\footnote{In practice, the trajectories need not be congruent provided they are not too different in length and a suitable range of speeds is available to the robots.}. We show in light gray the pairs of neighboring robots with an established communication link between them. Each part of the figure represents the state of the system at every quarter of the period required to complete a trajectory, starting with the state shown in Figure~\ref{fig:synchro_sample_a}.

This scenario arises naturally in missions of surveillance or monitoring \cite{acevedo_jint14,pasqualetti2012cooperative} and during structure assembly while robots are loading and placing parts in a structure \cite{assemblyStructure}, to name but two examples. However, the synchronization problem is an interesting problem in robotics and its proper solution will likely find applications beyond the ones considered here.

\cite{dbanez2015icra} solve the synchronization problem for a simplified model where all robots reliably travel
around unit circles at uniform speed (as in Figure~\ref{fig:synchro_sample}). 
They also discuss how to adapt the theory behind a simplified, not entirely practical model, to more general and realistic scenarios. \cite{Tro-paper} further addresses techniques to apply this synchronization model in realistic scenarios.
In these papers (\cite{dbanez2015icra,Tro-paper}), the authors also consider the possibility of a small number of dropouts and propose a protocol to minimize the detrimental effect that such failures may have on global system performance. In their proposal, the surviving robots handle a limited number of failures by ``shifting'' to a neighboring trajectory whenever the neighbor fails to arrive (see Figure~\ref{fig:ch_routes} for an illustration). For synchronization reasons, when a robot enters a neighboring trajectory $C$, it must follow the initial movement direction assigned to $C$. Also, during the shifting process, it must accelerate to maintain the schedule. Due to the kinematic constraints imposed by real scenarios, applying this recovery strategy, \cite{dbanez2015icra,Tro-paper} propose to assign opposite movement directions in neighboring trajectories, one clockwise (CW) and one counterclockwise (CCW). Consequently, the underlying communication graph must be bipartite. 

\begin{figure}
	\centering
	\includegraphics[scale=.8]{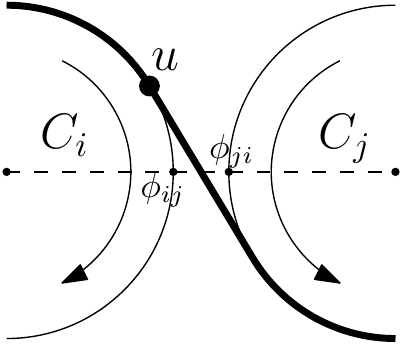}
	\caption{A robot shifts to a neighboring trajectory when it detects that the corresponding neighboring robot has left the system. The robot $u$ follows the path drawn with bold solid stroke. 
		There were no robots on the trajectory $C_j$ when $u$ is arriving at the link position in $C_i$.}
	\label{fig:ch_routes}
\end{figure}

In some cases, if enough robots leave the team, an undesirable phenomenon may occur: a robot, independent of how much longer stays in flight, it permanently fails to encounter other robots every time it arrives at a link, causing it to repeatedly shift to neighboring trajectories. In this case, we say that the robot is \emph{starving} or in \emph{starvation} mode. Figure~\ref{fig:starvation_intro} shows a synchronized system where two robots leave and the remaining robots, $u_1$ and $u_3$, permanently fail to encounter other robots at the link positions, so they enter the starvation state. 

\begin{figure}
	\begin{subfigure}{.245\textwidth}
		\centering
		\includegraphics{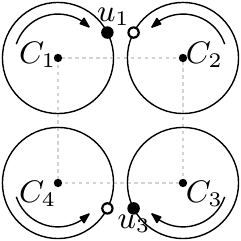}
		\caption{}
		\label{fig:starvation_intro_a}
	\end{subfigure}%
	\begin{subfigure}{.245\textwidth}
		\centering
		\includegraphics[page=2]{starvation_steps_intro.pdf}
		\caption{}
		\label{fig:starvation_intro_b}
	\end{subfigure}%
	\begin{subfigure}{.245\textwidth}
		\centering
		\includegraphics[page=3]{starvation_steps_intro.pdf}
		\caption{}
		\label{fig:starvation_intro_c}
	\end{subfigure} %
	\begin{subfigure}{.245\textwidth}
		\centering
		\includegraphics[page=4]{starvation_steps_intro.pdf}
		\caption{}
	\end{subfigure}
	\caption{When the two robots represented by hollow points leave the system, the surviving robots (solid points) follow the paths drawn with bold solid stroke.}
	\label{fig:starvation_intro}
\end{figure}

\subsection{Contribution and Organization}
\cite{andrew} introduces the notion of $k$-resilience in a synchronized system of mobile robots as a measure of the system's ability to gather information from a set of mobile robots with communication constraints. 
More formally, the {\em $k$-resilience} is defined as the smallest cardinality of a set of robots whose failure results in at least $k$ starving robots. 
Obviously, the larger the resilience, the more fault tolerant the system is. 
An $O(n^2)$ algorithm for computing the 1-resilience is proposed in \cite{andrew}. 

In this paper we address the problem of computing the $k$-resilience by formally establishing the concepts introduced by \cite{andrew} and \cite{dbanez2015icra,Tro-paper}. 
For simplicity, we focus on the combinatorial problem of computing the $k$-resilience of a synchronized system in the circular model (unit circles trajectories) considered by \cite{dbanez2015icra}. However, our results can be extended using the same arguments exposed by \cite{dbanez2015icra, Tro-paper}.

The contributions of this paper are summarized as follows:
\begin{itemize}
\item[$\bullet$] The next section states the problem of computing the $k$-resilience and describes useful properties of synchronized systems.
\item[$\bullet$] In Section~\ref{sec:hardness}, we show that, when $k$ is part of the input, computing the $k$-resilience of a synchronized system of robots for arbitrary communication graphs is NP-hard and show that the corresponding decision problem is NP-complete. 
\item[$\bullet$] Section~\ref{sec:computing} shows how to efficiently compute the $k$-resilience of a general synchronized system for small values of $k$.  Our time-complexity is $O(kn^{k+1})$.
\item[$\bullet$] In the same Section we prove that the $1$-resilience  problem can be solved  in almost linear time.
\item[$\bullet$] In Section~\ref{sec:trees} we include algorithms for the case in which the communication graph is a tree. This includes a linear algorithm for the $1$-resilience problem, an $O(t^2)$-time algorithm for the  $2$-resilience problem 
and an $O(tn^{k-1})$-time algorithm for the general problem, where $\sqrt{\pi n}/2-1\leq t\leq n-1$ is a parameter that depends on the topology of the tree.
\item[$\bullet$] Finally, in Section~\ref{ch:conclusions} we state two new open problems in the design of algorithms 
whose subquadratic solutions will imply more efficient algorithms for the resilience problems.
\end{itemize}

\section{Problem statement and preliminaries}
The simple model introduced by \cite{dbanez2015icra} is presented in this section using some formalisms that we will require later.

Let $T=\{C_1,\dots,C_n\}$ be a set of pairwise disjoint unit circles (trajectories). Let $\epsilon<0.5$ be the communication range of each robot in a team of $n$ robots, one per trajectory. The robots move at the same constant speed in the circles. The \emph{graph of potential links} of $T$ is the geometric graph $G_\epsilon(T) = (V,E_\epsilon)$ whose nodes are the centers of the circles in $T$ and whose edges are given by the pairs of circles centers at distance $2+\epsilon$ or less. The edge that connects a pair of adjacent trajectories $C_i$  and $C_j$ in $G_\epsilon(T)$ is denoted by $\{i,j\}$. Since communication is an important issue in cooperative scenarios, 
this work focuses on sets of trajectories whose graph of potential links is connected. Assume for the rest of the paper that we are working with a given set of trajectories $T$ and a fixed communication range $\epsilon<0.5$ such that the geometric graph $G_\epsilon(T)$ is connected.

For convenience, the position of a point in a circle is denoted by the angle measured from the positive horizontal axis. 
Assume, without loss of generality, that a trajectory can be covered by a robot in one time unit. The notion of \emph{schedule} was introduced in \cite{dbanez2015icra}. Below, we present a formal definition of this concept.

\begin{definition}[Schedule]
	Let $T=\{C_1,\dots,C_n\}$ be a set of trajectories. A \emph{schedule} on $T$ is a pair of functions $(f,g)$, $f:T\rightarrow[0,2\pi)$ and $g:T\rightarrow\{-1,1\}$, where $f(C_i)$ is the starting position in the circle $C_i$ and $g(C_i)$ is the movement direction in the circle $C_i$, 1 corresponds to CCW and $-1$ to CW. At an arbitrary time $t$, a robot's position in circle $C_i$ is: \begin{equation}\label{eq:position}
	f(C_i)+ 2\pi\cdot g(C_i)\cdot t.
	\end{equation}
\end{definition}

Let $T$ be a set of trajectories. A \emph{communication graph} $G=(V,E)$ on $T$ is a connected subgraph  of $G_\epsilon(T)$ with the same set of nodes and a subset of the edges ($E\subseteq E_\epsilon$). Two trajectories $C_i$ and $C_j$ are \emph{neighboring} in $G$ if $\{i,j\}\in E$. The \emph{link position} of $C_i$ with respect to $C_j$, denoted by $\phi_{ij}$, is the point of $C_i$ closest to $C_j$ (see Figure~\ref{fig:ch_routes}). The following definition presents the notion of synchronization using a schedule.

\begin{definition}[$G$-synchronized schedule]
	Let $T=\{C_1,\dots,C_n\}$ be a set of trajectories. Let $G=(V,E)$ be a communication graph on $T$. A schedule $(f,g)$ on $T$ is \emph{$G$-synchronized} if for all $\{i,j\}\in E$ there exists a value $t$ such that: $$f(C_i)+g(C_i)\cdot 2\pi\cdot t = \phi_{ij} \;\Leftrightarrow\; f(C_j)+g(C_j)\cdot 2\pi\cdot t = \phi_{ji}.$$
\end{definition}

The shifting protocol is formalized in terms of schedule as follows:
\begin{definition}[Shifting-protocol]
	Let $T=\{C_1,\dots,C_n\}$ be a set of trajectories. Let $G=(V,E)$ be a communication graph on $T$ and let $(f,g)$ be a $G$-synchronized schedule on $T$. Let $C_i$ and $C_j$ be neighboring trajectories in $G$. When a robot $u$ in $C_i$ arrives at the link position $\phi_{ij}$ at time $t$ and detects that there is no robot at $\phi_{ji}$, then $u$ \emph{shifts} to $C_j$ in order to assume the unattended task in $C_j$. During the shifting process it accelerates, such that, after a small time interval $\delta$, the robot will be in $C_j$ at position $f(C_j)+2\pi\cdot g(C_j)\cdot (t+\delta)$. After that, $u$ moves following the schedule in $C_j$, i.e., $u$ moves at the programmed constant speed in direction $g(C_j)$.
\end{definition}

Considering the kinematic constraints imposed by real scenarios in which the shifting-protocol is applied, \cite{dbanez2015icra} propose to use synchronized schedules where the neighboring robots are moving in opposite directions.

\begin{definition}[Synchronized communication system]
	
	Let $T=\{C_1, \dots, C_n\}$ be a set of trajectories. A \emph{synchronized communication system} (SCS) with communication graph $G=(V,E)$ on $T$ is a team of $n$ robots using  a $G$-synchronized schedule $(f,g)$ such that $g(C_i)=-g(C_j)$ for all $\{i,j\}\in E$. An $m$-partial SCS, $0<m\leq n$, is a synchronized communication system in which $n-m$ robots have left the team and the $m$ remaining robots apply the shifting strategy. 
\end{definition}

Notice that a SCS is only possible if the communication graph is bipartite and fulfills other properties exposed by \cite{dbanez2015icra}. Also, for every set of trajectories $T$ there exists a communication graph $G=(V,E)$ such that it is possible to make a $G$-synchronized schedule where  $g(C_i)=-g(C_j)$ for all $\{i,j\}\in E$. One possibility is to use the spanning tree of the graph of potential links of $T$ as the underlying communication graph.

Note that a SCS is a type of partial SCS where no robots have left. Thus, any claims about partial SCSs holds for SCSs as well.

\begin{definition}[Starvation]
	In an $m$-partial SCS, a robot \emph{starves} or is in \emph{starvation} if every time that it arrives at a link position the corresponding neighbor is not there, causing it to shift to the neighboring trajectory.
\end{definition}

\begin{definition}[$k$-resilience]
	The \emph{$k$-resilience} of a SCS ($k\geq 1$) is the minimum number of robots whose removal may cause the starvation of at least $k$ surviving robots. If it is not possible to obtain $k$ starving robots then the $k$-resilience is stated as infinity.
\end{definition}

This paper focuses on the following problem: 
\begin{problem}
	Given a SCS and a natural number $k$, determine the $k$-resilience of the SCS.
\end{problem}

Note that higher resilience values correspond to increased fault tolerance. To tackle this problem we need the notion of a \emph{ring}, first introduced in \cite{andrew}. 
In the sequel, we give useful properties of rings.
\begin{definition}[Ring]
	Let $T=\{C_1,\dots,C_n\}$ be a set of trajectories. A \emph{ring} in a SCS with communication graph $G$ on $T$, is the locus of points visited by a starving robot following the assigned movement direction in each trajectory and always shifting to the neighboring trajectory (in $G$) at the corresponding link positions.
\end{definition}

Each ring is a closed path composed of sections of various trajectories and has a direction of travel determined by the movement direction in the participating trajectories. Each section of a trajectory between two consecutive link positions participates in exactly one ring, thus, the rings in a SCS are pairwise disjoint.
Figure~\ref{fig:ring_samples} shows various SCSs with different numbers of rings.

\begin{figure}
	\centering
	\begin{subfigure}{0.2\textwidth}
		\centering
		\includegraphics[page=1, scale=0.5]{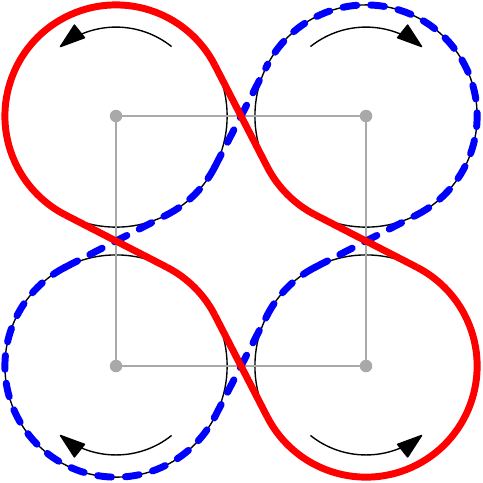}
		\caption{}
	\end{subfigure}%
	\begin{subfigure}{0.39\textwidth}
		\centering
		\includegraphics[page=2, scale=0.5]{rings-sample.pdf}
		\caption{}
	\end{subfigure}
	\begin{subfigure}{0.39\textwidth}
		\centering
		\includegraphics[page=3, scale=0.5]{rings-sample.pdf}
		\caption{}
	\end{subfigure}
	\caption{SCSs with two rings (a); one ring (b) and three rings (c).}
	\label{fig:ring_samples}
\end{figure}

\begin{definition}[Path in a ring]
	A \emph{path} in a ring $r$ from a point $p\in r$ to a point $q\in r$ is the ordered set of visited points from $p$ to $q$ following the travel direction of $r$ (it may contain tours on $r$). 
	If a path does not contain any tour in the ring then we say that it is a \emph{simple path}.
\end{definition}

As suggested by the examples in Figure~\ref{fig:ring_samples}, the lengths of rings in a system varies from ring to ring. In discussing the length of a ring, it is convenient to ignore the effect on distance arising from shifting between neighboring trajectories, i.e., to proceed as if neighboring circular trajectories were tangent to each other.

\begin{definition}[Length of a ring]\label{def:length_ring}
	The \emph{length of a ring} is defined as the sum of the lengths of the trajectory arcs forming the ring. Analogously, the \emph{length of a path in a ring} is defined as the sum of the lengths of the trajectory arcs (as many times as they are traversed) forming the path.
\end{definition}

Figure~\ref{fig:illustrating_length} illustrates the above definitions. The following proposition is a technical result needed to describe the length of a simple path between two robots in the same ring.

\begin{figure}
	\begin{subfigure}{.2\textwidth}
		\centering
		\includegraphics[page=3,scale=.75]{distance_same_ring.pdf}
		\caption{}
		
	\end{subfigure}\quad
	\begin{subfigure}{.5\textwidth}
		\centering
		\includegraphics[page=1,scale=.75]{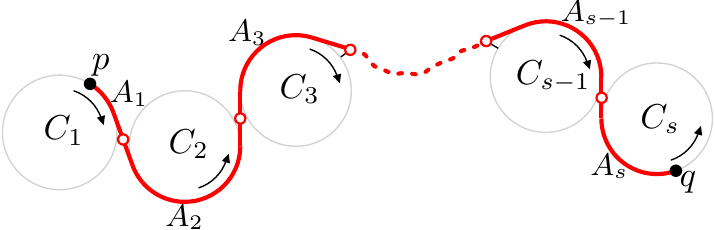}
		\caption{}
		\label{fig:dist_two_robots}
	\end{subfigure}
	\begin{subfigure}{.25\textwidth}
		\centering
		\includegraphics[page=2,scale=.75]{distance_same_ring.pdf}
		\caption{}
		\label{fig:twice_same_trajectory}
	\end{subfigure}
	
	\caption{(a) The length of a section $\sigma$ of a ring is the sum of the lengths of the arcs $A_1$ and $A_2$. (b) Path from a robot at position $p$ to a robot at position $q$ in the same ring. The second endpoint of $A_i$ and the first endpoint of $A_{i+1}$ are marked with a common non-solid point for all $1\leq i<s$. (c) A simple path between two robots in a ring that has two arcs in the trajectory $C$.}
	\label{fig:illustrating_length}
\end{figure}

\begin{proposition} \label{aux1}
	Let $\sigma$ be a simple path between two robots in a ring in an $m$-partial SCS at time $t$.
	Let $A_1,\dots, A_s$ denote the directed arcs traversed in $\sigma$ when following the travel direction of the ring, and let $C_i$ denote the trajectory containing $A_i$.\footnote{Note that a path between two robots may have two arcs in the same trajectory, see Figure~\ref{fig:twice_same_trajectory} for an example. Therefore, distinct indexes $i$ and $j$ may exist such that $C_i=C_j$. This does not affect the proof of the claim because it is not required that the arcs belong to different trajectories.}
	
	Then, for  all $1\leq j\le s$,
	\begin{equation}\label{eq:phi}
	f(C_j)+g(C_j)\cdot2\pi\cdot(t+\sum_{i=1}^j t_i)=\theta_{j}.
	\end{equation}
	where $t_i$ is the required time by a robot to traverse $A_i$ and
	$\theta_{j}$ is the angle position of the second endpoint of $A_j$ in $C_j$.
\end{proposition}

\begin{proof}
	We prove Equation~(\ref{eq:phi}) by induction on $j$.
	For $j=1$, by Equation~(\ref{eq:position}), we have the base case
	$$f(C_1)+g(C_1)\cdot 2\pi\cdot (t+t_1) = \theta_{1}.$$
	{\em Induction step}.  Suppose that Equation~(\ref{eq:phi}) holds for some $j<s$.
	By the definition of the synchronization schedule we have
	$$f(C_{j+1})+g(C_{j+1})\cdot 2\pi\cdot \left(t+\sum_{i=1}^{j} t_i\right)=
	\theta'_{j+1},$$
	where $\theta'_{j+1}$ is the position of the first endpoint of $A_{j+1}$ in $C_{j+1}$.
	Since $\theta_{j+1}=\theta'_{j+1}+g(C_{j+1})\cdot 2\pi \cdot t_{j+1}$, we have
	$$f(C_{j+1})+g(C_{j+1})\cdot 2\pi\cdot \left(t+\sum_{i=1}^{j+1} t_i\right)=\theta_{j+1},$$
	and the claim follows.
\end{proof}


The following lemma is established in \cite{andrew} with a slightly different argument.

\begin{lemma} \label{lem:robots_distance}
	In a partial SCS, the length of a path between any two robots in the same ring is in $2\pi\mathbb{N}$.
\end{lemma}
\begin{proof}
	Let $p$ and $q$ be the angle positions of two robots in the same ring at time $t$ and let $\sigma$ be the simple path in the ring from $p$ to $q$ as denoted in the above claim and shown in Figure~\ref{fig:dist_two_robots}. Thus, if $(f,g)$ is the synchronization schedule on the system, then we have
	
	\begin{eqnarray}
	f(C_1)+g(C_1)\cdot 2\pi\cdot t = p\label{eq:start_a}\\
	f(C_s)+g(C_s)\cdot 2\pi\cdot t = q\label{eq:end_b}
	\end{eqnarray}

	Then
	\begin{align*}
	f(C_{s})+g(C_{s})\cdot 2\pi\cdot\left(t+\sum_{i=1}^{s}t_i\right) &=q
	\text{\; (by Proposition~\ref{aux1})}\\
	f(C_{s})+g(C_{s})\cdot 2\pi\cdot t+g(C_{s})\cdot 2\pi\sum_{i=1}^{s}t_i&=q\\
	b+g(C_{s})\cdot 2\pi\sum_{i=1}^{s}t_i&=q \text{\; (by Equation~(\ref{eq:end_b}))}\\
	g(C_{s})\cdot 2\pi\sum_{i=1}^{s}t_i&=0
	\end{align*}
	
	Therefore, the angle $2\pi\sum_{i=1}^{s}t_i$ is in $2\pi\mathbb{Z}$.
	Since $2\pi\sum_{i=1}^{s}t_i$ is the length of the path $\sigma$, the lemma follows.
\end{proof}

\begin{corollary}\label{cor:ring_length}
	The length of every ring in a SCS is in $2\pi\mathbb{N}$.
\end{corollary}

\begin{proof}
	Let $r$ be a ring.
	If $r$ consists of a simple trajectory, then its length is $2\pi$.
	Suppose that $r$ consists of arcs from multiple trajectories.
	Let $p$ be a position of a robot in the ring, then the ring can be viewed as a closed path from $p$ to $p$.
	By Lemma \ref{lem:robots_distance} the length of the ring is in $2\pi\mathbb{N}$.
\end{proof}

The following remark is very useful to study the behavior of an $m$-partial SCS.

\begin{remark}\label{rem:crossing_point}
	Consider an $m$-partial SCS on a set of trajectories $T=\{C_1,\dots,C_n\}$. Let $\ell$ be a link between two neighboring trajectories $C_i$ and $C_j$ where two rings $r$ and $r'$ cross ($r$ and $r'$ could be the same ring).
	When a robot $u$ arrives at $\ell$ on $r$, there are two scenarios:
	\begin{itemize}
		\item If there is another robot $u'$ in the neighboring circle then, due to synchronization, $u'$ arrives at $\ell$ on $r'$ at the same time, and each robot keeps its trajectory but switches rings, see Figure~\ref{fig:cycle_invariant_with_neighbor}.
		\item If there is no robot in the neighboring circle then the robot $u$ shifts to the neighboring trajectory but remains in the same ring $r$, see Figure~\ref{fig:cycle_invariant_with_no_neighbor}.
	\end{itemize}
\end{remark}

\begin{figure}
	\begin{subfigure}{.5\textwidth}
		\centering
		\includegraphics[scale=.9, page=2]{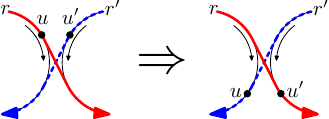}
		\caption{}
		\label{fig:cycle_invariant_with_neighbor}
	\end{subfigure}%
	\begin{subfigure}{.5\textwidth}
		\centering
		\includegraphics[scale=.9, page=3]{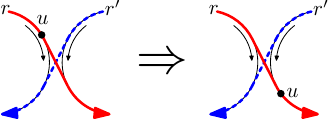}
		\caption{}
		\label{fig:cycle_invariant_with_no_neighbor}
	\end{subfigure}
	\caption{Transition of the position of the robots (represented by solid points) in the neighborhood of a  communication link. (a) Two robots $u$ and $u'$ arrive at a communication link at the same time, then they keep their trajectories. (b) Robot $u$ arrives at the link position and there is no robot in the neighboring trajectory, then $u$ shifts to the neighboring trajectory.}
	\label{fig:invariant_number_robots}
\end{figure}

\begin{lemma}\label{lem:ringCapacity}
	In an $m$-partial SCS the number of robots in a given ring remains invariant. If the length of the ring is $2l\pi$ then it has at most $l$ robots. Furthermore, in a SCS where no robots have left the system, a ring of length $2l\pi$ has exactly $l$ robots, each at distance $2\pi$ from the next.
\end{lemma}
\begin{proof}
	Notice that in an $m$-partial SCS a robot may change its ring only at the link positions, then from  Remark~\ref{rem:crossing_point}, the number of robots in a ring remains invariant. From Lemma~\ref{lem:robots_distance} and Corollary~\ref{cor:ring_length} we deduce that a ring of length $2l\pi$ has at most $l$ robots. We now prove the third claim. 
	Consider a system of $n$ trajectories and $m$ rings. Suppose that the $i$-th ring has length $2l_i\pi$ and $x_i$ robots, $1\le i\le m$. Then $x_i\le l_i$, for all $i$, and  
	$ n=\sum_{i=1}^m x_i$. Since the rings are disjoint, $\sum_{i=1}^ml_i=n$. Then
	\[ n=\sum_{i=1}^m x_i\le \sum_{i=1}^ml_i=n, \]
	and we conclude that $l_i=x_i$ for all $i$.
\end{proof}

The following notion gives us a useful tool to address the hardness of computing the $k$-resilience of a SCS.
\begin{definition}[Starvation number]
	The \emph{starvation number} of a SCS is the maximum possible number of starving robots in a partial SCS.
\end{definition}

The following result is deduced directly from the definitions of starvation number and $k$-resilience.
\begin{corollary}
	If the starvation number of a SCS is $s$, then the $k$-resilience of the system is \emph{infinity} for all $k>s$.
\end{corollary}

\begin{lemma}\label{lem:resilience_rel_starv}
	If the starvation number of a SCS is $s$ then the $s$-resilience of the system is $n-s$.
\end{lemma}
\begin{proof}
	Let $R$ be the $s$-resilience of the system, by definition $R\leq n-s$. Thus, there exists a set of $R$ robots whose removal induces the starvation of $s$ robots. Suppose $R<n-s$ then, in the resultant partial SCS, the set of non-starving live robots is not empty and its cardinality is greater than 1 by definition of starvation. Therefore the removal of all but one non-starving live robots results in a new partial SCS with $s+1$ starving robots, a contradiction.
\end{proof}


\section{Hardness of Computing the $k$-Resilience in General Graphs} \label{sec:hardness}
In this Section we prove that computing the $k$-resilience of a SCS is NP-hard in general. First, we introduce some notation. The middle point of the edge connecting the two closest points of two neighboring circles $C_i$ and $C_j$ is called a \emph{crossing point}. 
A starving robot may pass it in two possible directions, one following the ring from $C_i$ to $C_j$ and other following the ring from $C_j$ to $C_i$ (these two rings could be one and the same), see Figure~\ref{fig:invariant_number_robots}. We call them \emph{crossing directions}.

Let $p$ be a point of a ring $r$. Let $d$ be a non-negative real number. The point $p+d$ of $r$ is the point reached by traveling distance $d$ from $p$ in the travel direction of $r$. We say that the position $p+d$ is $d$ units \emph{ahead} of $p$. Analogously, the point $p-d$ of $r$ is the point $q$ such that $p$ is $d$ units ahead of $q$ in $r$.
We say that the position $p-d$ is $d$ units \emph{behind} $p$. Note that $d$ could be greater than the length of the ring $r$. In this case, the positions $p+d$ and $p-d$ are reached after one or more round trips in the ring. 

The following result is established in \cite{andrew}.
\begin{lemma}\label{lem:invariant_robot_pos_ring}
	Let $p$ be a point in a ring $r$ in an $m$-partial SCS and let $t>0$ be a real number. If there is a robot at $p$ at some time, then after time $t$ there will be a robot, not necessarily the same, at point $p+2t\pi$ of $r$. Also, $t$ units of time earlier there was a robot, not necessarily the same, at point $p-2t\pi$ of $r$.
\end{lemma}
\begin{proof}
	To prove the first claim, consider the path $\sigma$ from $p$ to $p+2t\pi$ in $r$.
	At each crossing point, there are two possible scenarios (see Remark~\ref{rem:crossing_point}). 
	If the neighbor does not appear at the crossing point, then 
	the same robot continues along the ring, as shown in  
	Figure~\ref{fig:cycle_invariant_with_no_neighbor}. 
	Otherwise, the neighboring robot will continue in $\sigma$ (in the neighboring trajectory) 
	as shown in  Figure~\ref{fig:cycle_invariant_with_neighbor}. 
	After time $t$, there will be a robot at point $p+2t\pi$ of $r$. 
	
	To prove the second claim, we apply the same argument backward. 
	Let $\sigma$ be the path from $p-2t\pi$ to $p$ in $r$.
	If $\sigma$ does not contain crossing points, then the robot at $p$ was at $p-2t\pi$ $t$ units of time earlier. If $\sigma$ contains crossing points then they can be traversed back similarly: either the same robot or the neighbor will be before each crossing point at position $p-2t\pi$ in $\sigma$ at $t$ units of time earlier.
\end{proof}



\begin{lemma}\label{lem:prevention_starving}
	In an $m$-partial SCS, let $r$ and $r'$ be rings (not necessarily distinct) that cross each other at a point $c$. Let $u$ and $u'$ be two robots in $r$ and $r'$, respectively. If there are two paths of equal lengths, one from $u$ to $c$ in $r$ (possibly longer than $r$) and other from $u'$ to $c$ in $r'$ (possibly longer than $r'$), then $u$ and $u'$ are not starving.
\end{lemma}
\begin{proof}
	Let $l$ be the length of the paths.
	After traveling distance $l$ from the current position of $u$, the resulting position is the point $c$. Traveling distance $l$ from the current position of $u'$, the resulting position is $c$ as well. We focus on proving that $u$ does not starve as the analysis for $u'$ is analogous.
	Let $t$ be the required time to travel $l$ units of length.
	If during the next $t$ units of time $u$ meets some robot, then $u$ is not starving. If after $t$ units of time $u$ did not meet any robot, then $u$ arrives at point $c$, by Lemma~\ref{lem:invariant_robot_pos_ring}, another robot in $r'$ arrives at $c$ too. Therefore, $u$ does not starve.
\end{proof}

The above lemma leads us to the following definition:
\begin{definition}
	In an $m$-partial SCS, let $u$ and $u'$ be two robots in rings $r$ and $r'$ (not necessarily distinct), respectively. We say that $u'$ \emph{prevents} $u$ from starving if there is a crossing point $c$ between $r$ and $r'$ such that there are two paths of equal lengths, one from $u$ to $c$ in $r$ (possibly longer than $r$) and other from $u'$ to $c$ in $r'$ (possibly longer than $r'$).
\end{definition}

Observe that if $u'$ prevents $u$ from starving then $u$ prevents $u'$ from starving. 

The following is claimed without a proof in \cite{andrew}.

\begin{corollary}\label{cor:prevention_from_starving}
In an $m$-partial SCS, a robot is starving if and only if all the robots that prevent it from starving have failed.
\end{corollary}

\begin{proof}
($\Rightarrow$) For the first implication, if a robot is in starvation then all the robots that prevent it from starving have failed, follows directly from Lemma~\ref{lem:prevention_starving}. 
	($\Leftarrow$) We prove the second implication by contradiction.
	Let $u$ be a robot on a ring $r$ such that, at some point in time $t_0$, all the robots that prevent it from starving have failed. Suppose that after traveling for $t$ units of time $u$ meets a robot $u'$ in a crossing point $c$ of $r$ with another ring $r'$ (the rings could be the same). By Lemma~\ref{lem:invariant_robot_pos_ring} there was a robot at position $c-2t\pi$ of the ring $r'$ at time $t_0$, and this robot prevents $u$ from starving, a contradiction. 
\end{proof}

\typeout{should we include crosses itself in a definition earlier? ML}
We have mentioned the fact that the two rings meeting at a crossing point could be one and the same, i.e., it is possible for a ring to ``cross itself''. Thus, we say that a ring $r$ \emph{crosses itself} if there is a crossing point which is traversed by $r$ in the two crossing directions, see Figure~\ref{fig:crossing_itself}. The following results focus on rings that cross themselves.

\begin{figure}
	\centering
	\includegraphics[scale=.8]{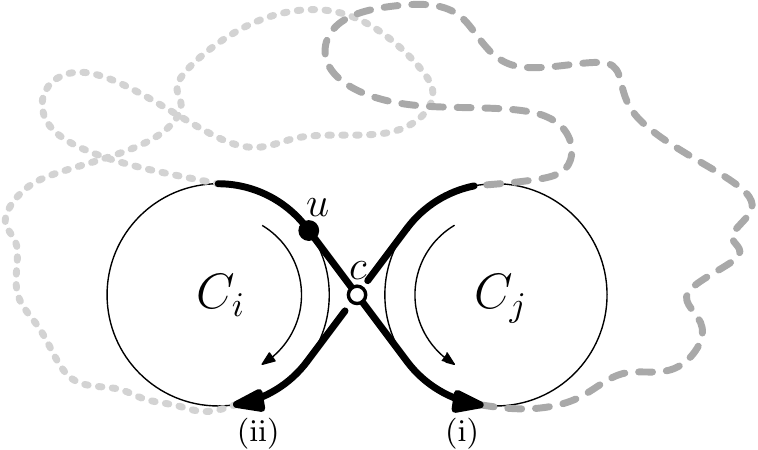}
	\caption{Ring that crosses itself. The directions of movement through the crossing point $c$ are represented with bold solid strokes.}
	\label{fig:crossing_itself}
\end{figure}

\begin{lemma}\label{lem:crossing_tie}
	Let $r$ be a ring that crosses itself at a point $c$ between circles $C_i$ and $C_j$. Every starving robot in $r$ passes through $c$ periodically and alternating the crossing directions.
\end{lemma}
\begin{proof}
	Let $r$ be a ring that crosses itself at $c$ between $C_i$ and $C_j$ (see Figure~\ref{fig:crossing_itself}). We show that if a starving robot $u$ in $r$ crosses $c$ from $C_i$ to $C_j$, then the next time that $u$ crosses $c$, it will do so from  $C_j$ to $C_i$, and vice versa. 
	Suppose that $u$ crosses $c$ from $C_i$ to $C_j$ following direction (i) in Figure~\ref{fig:crossing_itself}. Note that the ring may have other crossing points with itself. Obviously $u$ will return to the crossing point $c$ (because rings are closed paths), and there are only two ways to cross through $c$, (i) from $C_i$ to $C_j$ and (ii) from $C_j$ to $C_i$. Suppose, for contradiction, that the next time $u$ crosses $c$ it follows direction (i) again.
	Then, since a ring is a closed path, $u$ has completed a tour in the ring without using the crossing direction (ii). Therefore $r$ does not cross itself at $c$. This is a contradiction and the result follows.
\end{proof}

\begin{definition}[Tie of a ring]
	A \emph{tie} of a ring is a closed path that starts and ends at a crossing point of the ring with itself (without passing through this crossing point). 
\end{definition}

From the synchronization between robots in neighboring circles and using Lemma~\ref{lem:robots_distance} the following corollary is deduced:
\begin{corollary}\label{cor:tie_length}
	A crossing point of a ring with itself determines two ties. Moreover, the length of a tie is in $2\pi\mathbb{N}$.
\end{corollary}

Figure~\ref{fig:crossing_itself} shows the two ties determined by the crossing point $c$: the one represented with dashes follows direction (i) and the other (shown with dots) follows direction (ii).

\begin{lemma}\label{lem:tree_one_ring}
If the communication graph of a SCS is a tree then there is a single ring.
\end{lemma}
\begin{proof}
We prove the lemma by induction on the number of trajectories. Clearly, if there is only one trajectory, the ring is unique.
	
	\begin{figure}
		\centering
		\begin{subfigure}{.48\textwidth}
			\centering
			\includegraphics{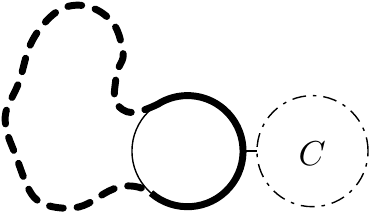}
			\caption{}\label{fig:tree_ring1}
		\end{subfigure}%
		\begin{subfigure}{.48\textwidth}
			\centering
			\includegraphics[page=2]{tree_ring.pdf}
			\caption{}\label{fig:tree_ring2}
		\end{subfigure}
		\caption{(a) Ring corresponding to $T'$.
			(b) Ring corresponding to $T$.}
		\label{fig:tree_ring}
	\end{figure}
	
	Suppose that the claim holds for any tree with $n$ trajectories. We show that it also holds for any tree $T$ with $n+1$ trajectories. Let $C$ be a trajectory corresponding to a leaf in $T$, see Figure~\ref{fig:tree_ring1}. Let $T'$ be the tree obtained by deleting trajectory $C$. Then
	there is exactly one ring corresponding to $T'$. Adding $C$ to the system, the ring changes by adding a loop covering $C$ as shown in Figure~\ref{fig:tree_ring2} and the lemma follows.
\end{proof}

From Lemma~\ref{lem:tree_one_ring} and Corollary~\ref{cor:tie_length} we have:
\begin{corollary}\label{cor:crossing_point_trees}
	In a SCS of $n$ trajectories whose communication graph is a tree, a crossing point determines two ties of lengths $2l\pi$ and $2(n-l)\pi$ respectively, where $l\in\mathbb{N}$.
\end{corollary}

\begin{figure}
	\centering
	\begin{subfigure}{.48\textwidth}
		\centering
		\includegraphics[scale=0.7]{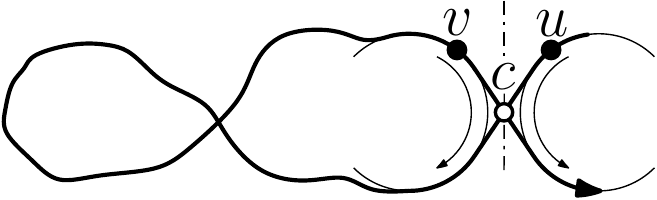}
		\caption{}
		\label{fig:robots_in_crossing_tie}
	\end{subfigure}
	\begin{subfigure}{.48\textwidth}
		\centering
		\includegraphics[scale=0.7]{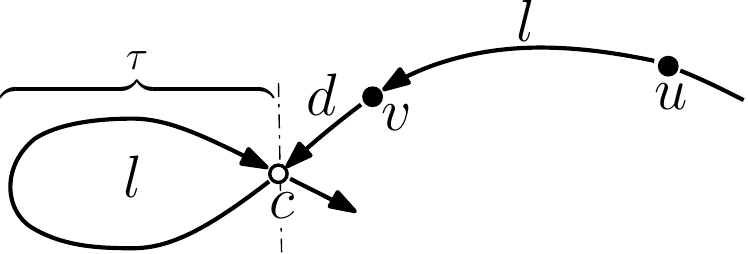}
		\caption{}
		\label{fig:prevent_in_tie}
	\end{subfigure}
	\caption{(a) The length of the tie to the left of the dash-dotted vertical line is equal to the length of the section of ring from $u$ to $v$. (b) The robots $u$ and $v$ prevent each other from starving because the distance between them is equal to the length of a tie in the same ring.}
\end{figure}

\begin{lemma}\label{lem:starve_prevention_ring}
	Let $u$ and $v$ be robots on a ring $r$ of an $m$-partial SCS. Then $u$ prevents $v$ from starving if and only if $r$ has a tie whose length is equal to the length of a simple path between $u$ and $v$.
\end{lemma}
\begin{proof}
	($\Rightarrow$) 
	Suppose that $u$ prevents $v$ from starving and assume that we remove all live robots in the system except for $u$ and $v$. Then $u$ and $v$ must meet each other after a while at a crossing point $c$ of $r$ with itself. 
	Observe that the path from $u$ to $v$ is one of the ties determined by $c$, see Figure~\ref{fig:robots_in_crossing_tie}, so the length of this tie is equal to the length of the path between $u$ and $v$.
	($\Leftarrow$) 
	Suppose $r$ has a tie whose length is equal to the length $l$ of the path from $u$ to $v$. Let $c$ be a crossing point that determines a tie $\tau$ of length $l$, see Figure~\ref{fig:prevent_in_tie}. There are two ways to reach $c$, one entering and the other leaving $\tau$. Let $d$ be the length of the path from $v$ to $c$ entering $\tau$. The point obtained by traveling $l+d$ units of length from the current positions of $u$ and $v$, respectively, is $c$ in both cases. Consequently, $u$ and $v$ prevent each other from starving.
\end{proof}

From lemmas~\ref{lem:starve_prevention_ring} and \ref{lem:tree_one_ring}, and Corollary~\ref{cor:prevention_from_starving} we deduce:
%
\begin{corollary}\label{cor:starvaion_tree}
In an $m$-partial SCS whose communication graph is a tree, a robot $u$ is starving if and only if the distance between $u$ and any other live robot is different from the lengths of all ties in the single ring of the system.
\end{corollary}

We now show that rings are related to \emph{circulant graphs}.
A graph on $n$ nodes is \emph{circulant} if the nodes of the graph can be numbered from 0 to $n-1$ such that, if two nodes $x$ and $(x+d)\bmod{n}$ are adjacent, then two nodes $z$ and $(z+d) \bmod{n}$ are adjacent for any $z$.
We call such a node numbering a \emph{c-order}.

Let $r$ be a ring of length $2m\pi$ containing the $m$ surviving robots of an $m$-partial SCS. Let $0,1,\dots, m-1$ be a circular enumeration of the robots in $r$, following the travel direction of the ring. Lemma~\ref{lem:ringCapacity} implies that robot $i$ is $2\pi$ units ahead of robot $i-1$ (mod $m$, as usual). Let $2l_1\pi,\dots, 2l_t\pi$ be the lengths of all ties in $r$. By Lemma~\ref{lem:starve_prevention_ring}, robot $i$ starves if and only if \textit{all} the robots that prevent it from starving fail. Note that these are the robots with indices in $\{i+l_1,\dots, i+l_t\}$. The relation ``prevent from starving'' between the robots of $r$ can be modeled using an undirected graph whose nodes correspond to the robots in the ring and, for all $i\ne j$, there is an edge between nodes $i$ and $j$ if and only if
robots $i$ and $j$ prevent each other from starving. The resulting graph is circulant.

\begin{figure}
	\centering
	\includegraphics[page=2, scale=0.7]{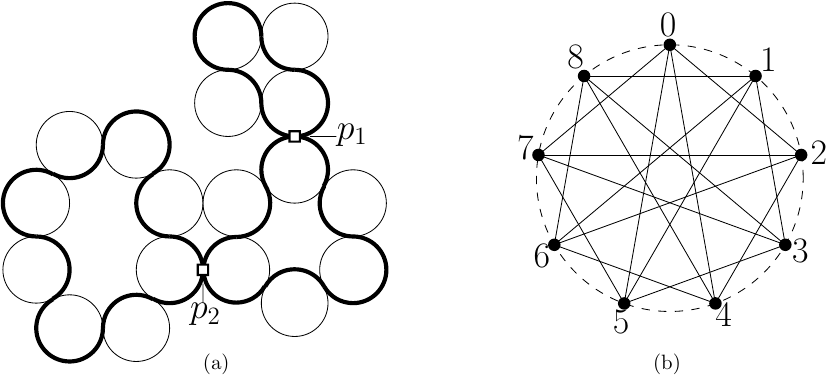}
	\caption{(a) Example of a SCS. With bold solid stroke a ring that crosses itself twice in $p_1$ and $p_2$. (b) The circulant graph that models the relation ``prevent from starving'' corresponding to the ring in (a).}
	\label{fig:ring_to_circulant}
\end{figure}

Figure~\ref{fig:ring_to_circulant}a shows a SCS with a ring in bold stroke of length $18\pi$. It has two crossing points with itself: $p_1$ determines ties of length $4\pi$ and $14\pi$, and $p_2$ determines ties of length $8\pi$ and $10\pi$. Therefore, enumerating the robots of this ring from $0$ to $8$ in the travel direction of the ring, we see that robot $i$ is prevented from starving by robots $i+2, i+4, i+5$ and $i+7$, (mod $9$). Figure~\ref{fig:ring_to_circulant}b shows the corresponding circulant graph.

From Lemma~\ref{lem:starve_prevention_ring} and Corollary~\ref{cor:prevention_from_starving} the following result is obtained:
\begin{corollary}\label{cor:circulantgraph-starving}
	The maximum possible number of starving robots in a ring is equal to the cardinality of the maximum \emph{independent set}\footnote{Subset of nodes in a graph that does not contain two adjacent nodes.} in the corresponding circulant graph.
\end{corollary}

In the following, we define an auxiliary operation to transform a circulant graph into another circulant graph with some interesting properties for us.

\begin{definition}[$K_{n,n}$-augmentation]
	Let $G=(V,E)$ be a graph with $n$ nodes. 
	A graph $G'=(V',E')$ is a \emph{clone} of $G$ if $G'$ and $G$ are isomorphic and $V\cap V'=\emptyset$. 
	The \emph{$K_{n,n}$-augmentation} of $G$, denoted by $\overline{G}=(\overline{V},\overline{E})$, is the graph resulting from a graph join operation between $G$ and a clone $G'$, i.e 
	$\overline{V}=V\cup V'$ and $\overline{E} = E\cup E'\cup \{\{v,w\}\;\vert\; v\in V, w\in V'\}$.
\end{definition}

From now on we denote a vertex in a graph of $n$ vertices by $v_i$ with $i\in\{0,\dots,n-1\}$. In general, the vertex indices are taken modulo $n$.

The following result can be deduced directly from the definition of circulant graphs.

\begin{lemma}\label{lem:connected_components}
	A graph is circulant if and only if all its connected components are isomorphic to the same circulant graph.
\end{lemma}
\begin{proof}
	$(\Leftarrow)$ Let $G=(V,E)$ be a circulant graph of $n$ nodes and let $v_0,\dots,v_{n-1}$ be a c-order of $G$. 
	Construct a graph $G'$ as the union of $m$ disjoint clones of $G$ and let $v^{(i)}_0,\dots,v^{(i)}_{n-1}$ be the c-order of $i$th clone corresponding to the c-order of $G$.
	It is easy to see that
	\[ v^{(1)}_0,v^{(2)}_0,\dots,v^{(m)}_0,v^{(1)}_1,v^{(2)}_1,\dots,v^{(m)}_1,\dots,v^{(1)}_{n-1},v^{(2)}_{n-1},\dots,v^{(m)}_{n-1} \] is a c-order of $G'$. Therefore $G'$ is circulant.
	
	$(\Rightarrow)$ Let $G=(V,E)$ be a circulant graph and let $v_0,\dots,v_{n-1}$ be its c-order. 
	Let $C$ be a connected component of $G$ containing $v_0$. Let $m>0$ be the lowest value such that $v_m\in C$. Then the nodes $v_{2m},v_{3m},\dots$ are in $C$. It can be proven that $m$ divides $i$ for all $v_i\in C$. It can also be proven that $m$ divides $n$. Therefore $C=\{v_0,v_m,v_{2m},\dots,v_{n-m}\}$. From here, it is easy to see that $G$ has $m$ isomorphic connected components of the form $\{v_i,v_{i+m},v_{i+2m},\dots,v_{n-m+i}\}$ for all $0\leq i<m$.
\end{proof}

\begin{lemma}\label{lem:circ_augmCirc}
	Let $G=(V,E)$ and $\overline{G}=(\overline{V},\overline{E})$ be a graph and its $K_{n,n}$-augmentation, respectively. $G$ is a circulant graph if and only if $\overline{G}$ is a circulant graph.
\end{lemma}
\begin{proof} 
	Let $G'=(V,E)$ be the clone of $G$ used in the creation of $\overline{G}$.
	
	
	($\Rightarrow$) 
	Let $v_0,\dots,v_{n-1}$ be a c-order of $G$ and 
	$v'_0,\dots,v'_{n-1}$ be the corresponding c-order of $G'$.
	Consider the ordering $L=(v_0,v'_0,v_1,v'_1,\dots,v_{n-1},v'_{n-1})$
	of $\overline{V}$.
	We show that $L$ is a c-order of $\overline{G}$. 
	Indeed, if $d$ is odd then, for any $i$, $(v_i,v_{i+d})$ is an edge of $\overline{G}$.
	If $d$ is even, then both $v_i$ and $v_{i+d}$ are in the same graph $G$ or $G'$. 
	Then $(v_i,v_{i+d})\in \overline{E}$ if and only if $(v_0,v_{i+d/2})\in E$. 
	Therefore the graph $\overline{G}$ is circulant.
	
	
	$(\Leftarrow)$ 
	If $\overline{G}$ is circulant then its complement graph $\neg\overline{G}$ is circulant. 
	By Lemma~\ref{lem:connected_components}, $\neg\overline{G}$ has $m$ 
	isomorphic components. 
	Since $\neg\overline{G}$ is the union of 
	$\neg G$ and $\neg G'$, $\neg G$ has $m/2$ isomorphic components that are circulant graphs.  
	Thus, $G$ is a circulant graph by Lemma~\ref{lem:connected_components}.
\end{proof}

\begin{lemma}\label{lem:ind_set_augmentation}
	Let $G=(V,E)$ and $\overline{G}=(\overline{V},\overline{E})$ be a graph and its $K_{n,n}$-augmentation, respectively.
	The maximum independent set of $G$ and the maximum independent set of $\overline{G}$ have the same cardinality.
\end{lemma}
\begin{proof}
	Let $H\subseteq V$ and $\overline{H}\subseteq \overline{V}$ be maximum independent sets in $G$ and $\overline{G}$, respectively. Notice that the vertices in $H$ also form an independent set in $\overline{G}$, thus $|H|\leq |\overline{H}|$. 
	Since $\overline{G}$ is the $K_{n,n}$-augmentation, $\overline{H}$ cannot 
	contain a vertex from $V$ and a vertex from $V'$. 
	So, either $\overline{H}\subseteq V$ or $\overline{H}\subseteq V'$.
	Then $|\overline{H}|\leq |H|$ and $|H|=|\overline{H}|$.
\end{proof}

\begin{note} 
A circulant graph $G=(V,E)$ of $n$ nodes labeled $v_0,\dots,v_{n-1}$ 
	can be shortly denoted as $C_nS$ where $S=\left\lbrace d\in\mathbb{N}\; \middle| \{v_i,v_{i+d}\}\in E, 1\leq d\leq \left\lfloor\dfrac{n}{2}\right\rfloor\right\rbrace$ is the set of ``jumps'' adjacent vertices. See Figure~\ref{fig:circulant_augmentation}, for two examples of this notation.
\end{note}

Notice that for every pair of values $i$ and $j$ such that $0\leq i<j<n$, if $\{v_i,v_j\}\in E$ then there exists $d\in S$ such that $i+d=j$ or \mbox{$j+d\equiv i\pmod{n}$}.
Thus the $K_{n,n}$-augmentation of $C_nS$ can be denoted by $C_{2n}\overline{S}$ where  $$ \overline{S}=\left\lbrace2d\;\middle|\;d\in S\right\rbrace\cup\left\lbrace2i-1\;\middle|\;1\leq i \leq \left\lfloor\dfrac{n+1}{2}\right\rfloor\right\rbrace.$$ 
Figure~\ref{fig:circulant_augmentation} shows an example of a circulant graph and its $K_{n,n}$-augmentation. Notice that the set of jumps of the $K_{n,n}$-augmentation of $C_nS$ contains all the odd numbers in the interval $[1,n]$.

\begin{figure}
	\centering
	\begin{subfigure}{.48\textwidth}
		\centering
		\includegraphics[page=2, scale=.4]{k_nn_augmentation.pdf}
		\caption{}
	\end{subfigure}%
	\begin{subfigure}{.48\textwidth}
		\centering
		\includegraphics[page=1, scale=.4]{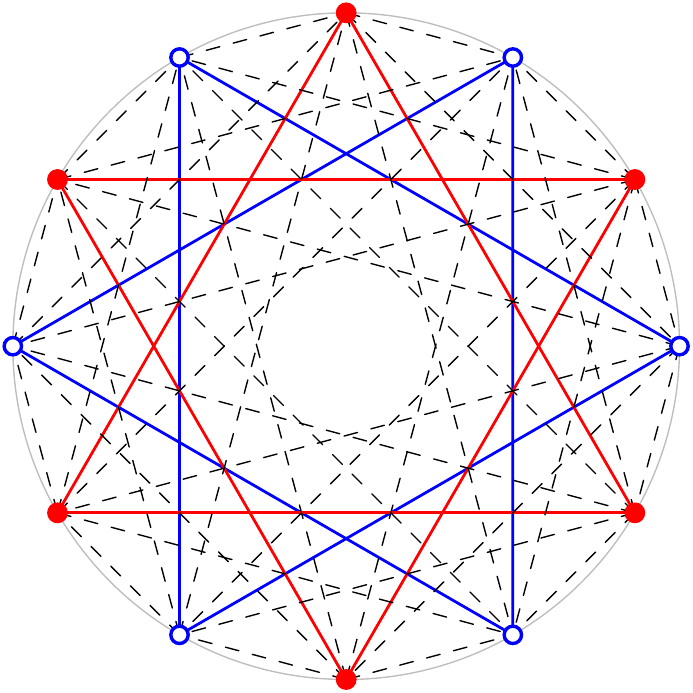}
		\caption{}
	\end{subfigure}%
	\caption{(a) Circulant graph $C_6\{2\}$. (b) Circulant graph $C_{12}\{1,3,4,5\}$ which is the \text{$K_{6,6}$-augmentation} of $C_6\{2\}$, with solid points denoting the nodes of the original graph, and non-solid ones, the vertices of its clone. Original and cloned vertices are connected with dashed edges.}
	\label{fig:circulant_augmentation}
\end{figure}

We are ready to prove the main result of this section.

\begin{theorem}\label{thm:hardness}
	The problem of computing the starvation number of a SCS (SN-SCS) is NP-hard, even, if the communication graph of the SCS is a \emph{caterpillar tree}\footnote{A caterpillar tree is a tree in which all the vertices are within distance 1 of a central path.}. 
\end{theorem}
\begin{proof}
	We use a reduction from the problem of computing the
    maximum independent set in a circulant graph \textsc{(MIS-CG)} which is NP-hard \cite{Codenotti1998123}.
	Let $C_nS$ be a circulant graph with $n\geq 2$, as input to the MIS-CG problem. For convenience we work with $C_{2n}\overline{S}$ 
    which is the $K_{n,n}$-augmentation of the given circulant graph $C_nS$. Recall that the problem of computing a maximum independent set for $C_nS$ is equivalent to the problem of computing a maximum independent set for $C_{2n}\overline{S}$ and that $\overline{S}$ contains all odd numbers in $[1,n]$.
	
	By Corollary~\ref{cor:circulantgraph-starving}, it suffices to transform $C_{2n}\overline{S}$ into a SCS of $2n$ circles whose communication graph is a caterpillar tree such that
	\begin{equation}\label{SCS_statement}
	d\in \overline{S} \text{ \quad if and only if \quad there is a tie of length }2d\pi 
	\text{ in the SCS}.  
	\end{equation}
	
	
	We place the circles on three horizontal lines with coordinates in 1, 0 and $-1$ as illustrated in Figure~\ref{fig:SCS_construction}.
	First, place the circle $C_0$ on the 0-line. Then place $C_i, i=1,\dots,n$ as follows. Let $C_j$ be the last circle placed on the 0-line. 
	\begin{enumerate}
		\item If $i\in \overline{S}$ then add the circle $C_i$ to the 0-line touching $C_j$, see Figure~\ref{fig:SCS_construction}a.
		\item If $i\notin \overline{S}$ then add the circle $C_i$ touching $C_j$ but alternating between centered on the
		1-line and centered on the $-1$-line. In other words,
		if the last added circle not centered on the 0-line is centered on the 1-line, then center $C_i$ on the $-1$-line, and vice-versa. see Figures \ref{fig:SCS_construction}b and \ref{fig:SCS_construction}c, respectively. 
	\end{enumerate}
	
	Notice that $i$ in the second case is even since $\overline{S}$ contains all odd numbers in $[1,n]$. 
	Thus, the next circle $C_{i+1}$ will be placed on the 0-line. 
	Since the lines 1 and $-1$ alternate, $C_i$ touches only one circle, $C_j$. 
	
	\begin{figure}
		\centering
		\includegraphics[width=\textwidth]{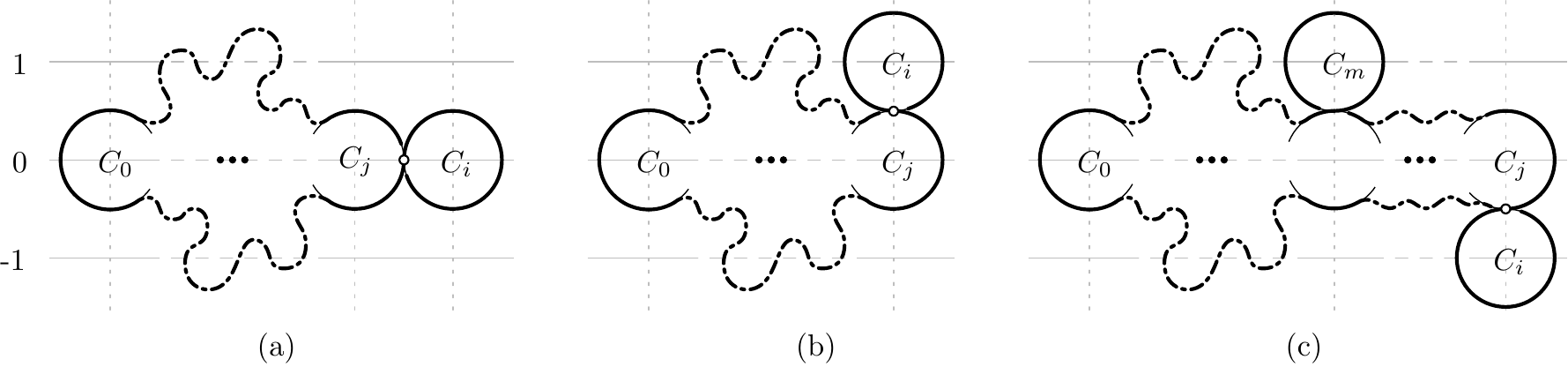}
		\caption{Addition of $C_i$ to the SCS. The ring of the SCS is shown with bold stroke. 
			The case of $i\in \overline{S}$ is shown in (a). The case of $i\notin \overline{S}$ is shown in (b) and (c). 
			(b) $C_i$ is added to the 1-line if $i$ is the smallest number not in $\overline{S}$ or $C_m$ is centered on the $-1$-line. 
			(c) $C_i$ is added to the $-1$-line if $C_m$ is centered on the 1-line.} 
		\label{fig:SCS_construction}
	\end{figure}	
	
	\begin{figure}
		\centering
		\begin{subfigure}{\textwidth}
			\flushleft
			\includegraphics[page=4,scale=0.7]{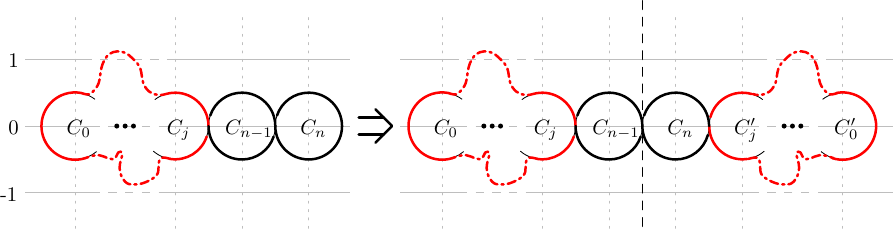}
			\caption{}\label{fig:completingSCS1}
		\end{subfigure}
		\begin{subfigure}{\textwidth}
			\flushleft
			\includegraphics[page=5,scale=0.7]{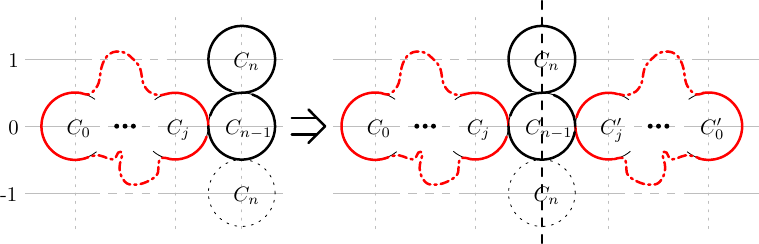}
			\caption{}\label{fig:completingSCS2}
		\end{subfigure}		
		\begin{subfigure}{\textwidth}
			\flushleft
			\includegraphics[page=6,scale=0.7]{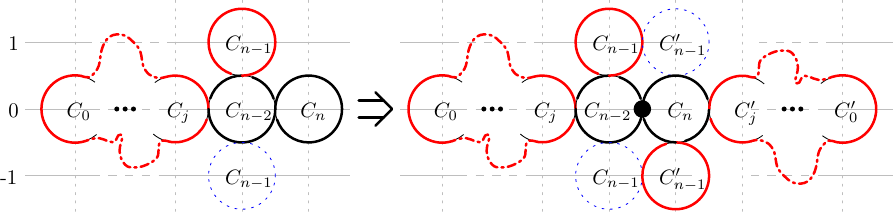}
			\caption{}\label{fig:completingSCS3}
		\end{subfigure}
		\caption{Instructions of how to add the $n-1$ remaining circles. 
			(a) If $C_{n-1}$ and $C_n$ are both on the 0-line, then apply symmetry about the vertical line between $C_{n-1}$ and $C_n$.
			(b) If $C_{n-1}$ is on the 0-line but $C_n$ is not, then apply symmetry about the vertical line passing through the center of $C_{n-1}$.
			(c) In the remaining case apply symmetry about the touching point of $C_{n-2}$ and $C_n$.}
	\end{figure}

	We have placed $n+1$ circles $C_0,\dots,C_n$. In order to add the $n-1$ remaining circles we proceed as follows:
	\begin{itemize}
		\item if $n$ is even then:
		\begin{itemize}
			\item if $n\in\overline{S}$ then we proceed as shown in Figure~\ref{fig:completingSCS1}.
			\item if $n\notin\overline{S}$ then we proceed as shown in Figure~\ref{fig:completingSCS2}.
		\end{itemize}
		\item if $n$ is odd then:
		\begin{itemize}
			\item if $(n-1)\in \overline{S}$ then we proceed as shown in Figure~\ref{fig:completingSCS1}.
			\item if $(n-1)\notin \overline{S}$ then we proceed as shown in Figure~\ref{fig:completingSCS3}.
		\end{itemize}
	\end{itemize}

	Now, we are ready to prove statement (\ref{SCS_statement}) on the obtained SCS.
	
	($\Rightarrow$) 
	If $d\in\overline{S}$, then $C_d$ is centered on the $0$-line and the tie determined by the crossing point between $C_d$ and the previous circle centered on the 0-line covers the $d$ circles to the left of $C_d$. Consequently, the length of this tie is $2d\pi$.
	
	($\Leftarrow$) Every crossing point between circles centered on the same vertical line determines two ties of length $2\pi$ and $2(2n-1)\pi$, respectively, and 1 is in $\overline{S}$. Consider the crossing point between two circles centered on the 0-line. 
	By symmetry, we can assume that the circles are $C_j$ and $C_i$ where $0\le j<i\le n$. 
	The crossing point determines a tie of length $2i\pi$ covering the $i$ circles to the left of $C_i$. 
	Since $C_i$ is on the 0-line, $i\in\overline{S}$. This argument completes the proof.
	
	Figure~\ref{fig:constructionsSamples} shows some examples of the SCS construction.
\end{proof}

	\begin{figure}
		\centering
		\begin{subfigure}{\textwidth}
			\centering
			\includegraphics[scale=.6]{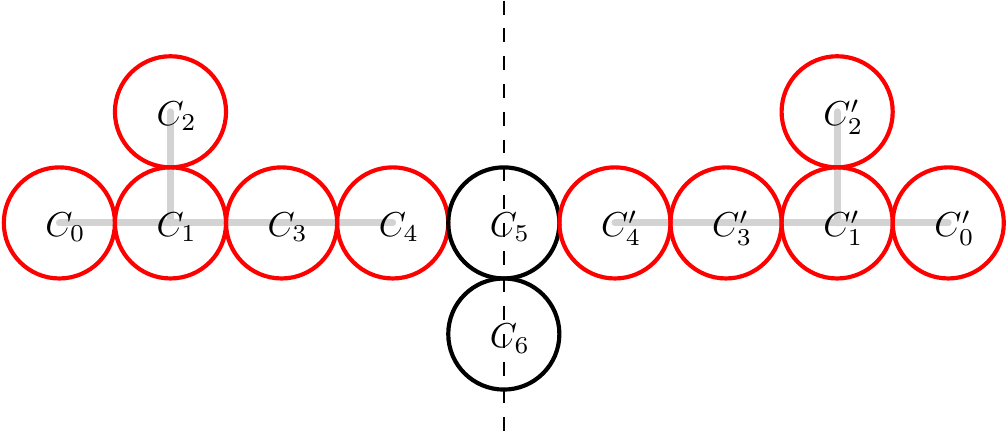}
			\caption{SCS obtained from the circulant graph $C_{12}\{1,3,4,5\}$ shown in Figure~\ref{fig:circulant_augmentation}b. }
		\end{subfigure}\\
		\vspace*{.5cm}
		\begin{subfigure}{.5\textwidth}
			\vspace*{.75cm}
			\centering
			\includegraphics[page=2, scale=.7]{circulant_to_SCS_samples.pdf}
			\caption{$C_4\{2\}$ and its $K_{4,4}$-augmentation: $C_8\{1,3,4\}$.}
		\end{subfigure}%
		\begin{subfigure}{.5\textwidth}
			\centering
			\includegraphics[page=3, scale=.7]{circulant_to_SCS_samples.pdf}
			\caption{SCS obtained from $C_8\{1,3,4\}$.}
		\end{subfigure}\\
		\vspace*{.5cm}
		\begin{subfigure}{\textwidth}
			\centering
			\includegraphics[page=4,scale=.6]{circulant_to_SCS_samples.pdf}
			\caption{$C_9\{3\}$ and its $K_{9,9}$-augmentation: $C_{18}\{1,3,5,6,7,9\}$.}
		\end{subfigure}\\
		\vspace*{.5cm}
		\begin{subfigure}{\textwidth}
			\centering
			\includegraphics[page=5,scale=.7]{circulant_to_SCS_samples.pdf}
			\caption{SCS obtained from $C_{18}\{1,3,5,6,7,9\}$.}
		\end{subfigure}
		\caption{Examples of construction of a SCS from the $K_{n,n}$-augmentation of some circulant graphs. The samples in (a) and (c) are obtained by applying the steps illustrated in Figure~\ref{fig:completingSCS2} and \ref{fig:completingSCS1}, respectively. The example in (e) is obtained by applying the steps illustrated in Figure~\ref{fig:completingSCS3}.}
		\label{fig:constructionsSamples}
	\end{figure}

The following result is deduced from Theorem~\ref{thm:hardness} and Lemma~\ref{lem:resilience_rel_starv}.
\begin{corollary}
	The problem of computing the $k$-resilience of a SCS is NP-hard.
\end{corollary}

In the rest of this section we focus on how to count the number of starving robots in an $m$-partial SCS in order to prove that the decision version of $k$-resilience problem is NP-Complete. 

Lemma~\ref{lem:starve_prevention_ring} gives us a method to check if two robots prevent each other from starving if they are in the same ring. But, how does one check if two robots prevent each other from starving if they are in different rings?

\begin{figure}
	\centering
	\includegraphics[scale=.8]{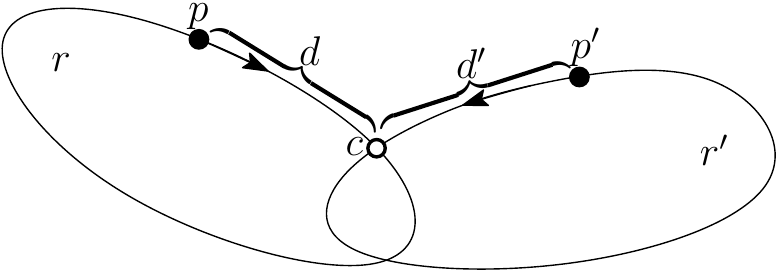}
	\caption{Illustration of Theorem~\ref{thm:NP_check}.}
	\label{fig:intersected_rings}
\end{figure}

\begin{theorem}\label{thm:NP_check}
	In an $m$-partial SCS, let $p$ and $p'$ be the positions of two robots $u$ and $u'$ in different rings $r$ and $r'$, respectively. Let $c$ be a crossing point between $r$ and $r'$. Let $d$ and $d'$ denote the lengths of the simple paths from $p$ and $p'$ to $c$, respectively. Let $2 l\pi$ and $2 l'\pi$ be the lengths of $r$ and $r'$ respectively. Then:
	\begin{itemize}
		\item[$\bullet$] $d-d'$ is in $2\pi\mathbb{Z}$.
		\item[$\bullet$] Let $s\in\mathbb{Z}$ such that $d-d'=2\pi s$. Then, the robots $u$ and $u'$ prevent each other from starving if and only if the \emph{greatest common divisor} of $l$ and $l'$ divides $s$. 
	\end{itemize}
\end{theorem}
\begin{proof}
Figure~\ref{fig:intersected_rings} shows the positions of the robots.
First, we prove that $(d-d')\in 2\pi\mathbb{Z}$. 
Suppose, w.l.o.g., that $d\leq d'$. 
Consider the time when the robot in $r$ reaches $c$. The robot in $r'$ has
distance $d'-d$ to $c$. By Lemma~\ref{lem:robots_distance} the distance between the two robots in $r'$ is $d'-d\in 2\pi\mathbb{N}$.
	
We focus now on the second claim. If $u$ and $u'$ prevent each other from starving then there are two paths of equal length, say $L$, from $p$ and $p'$ to $c$. The path from $p$ (resp. $p'$) to $c$ can be decomposed in the section from $p$ (resp. $p'$) to $c$ and zero or more round trips in $r$ (resp. $r'$). Let $x$ and $y$ denote the number of round trips of the paths in $r$ and $r'$, respectively. Therefore $L=d+x \cdot 2\pi l$ and $L=d'+y \cdot 2\pi l'$ where $(x,y)\in \mathbb{N}\times\mathbb{N}$. Then:
	\begin{eqnarray}
	d+x \cdot 2\pi l &=& d'+y\cdot2\pi l'\nonumber\\
	d-d' &=& y\cdot2\pi l'-x \cdot 2\pi l\nonumber\\
	2\pi s& =& y\cdot2\pi l'-x \cdot 2\pi l\nonumber\\
	s&=& y\cdot l'-x \cdot  l\label{eq:diophantine}
	\end{eqnarray}
	Considering $x$ and $y$ variables, this is a \emph{Diophantine equation} and has a solution where $x$ and $y$ are integers if and only if $\gcd(l,l')\mid s$.
	
	($\Rightarrow$) If $u$ and $u'$ prevent each other from starving, then there exists a solution $(x_1,y_1)\in \mathbb{N}\times\mathbb{N}$ for Equation~(\ref{eq:diophantine}). Therefore, $\gcd(l,l')\mid s$.
	
	($\Leftarrow$) If $\gcd(l,l')\mid s$ then there exist infinitely many solutions for Equation~(\ref{eq:diophantine}) in $\mathbb{Z}\times\mathbb{Z}$. Observe that the line represented by Equation~(\ref{eq:diophantine}) has positive slope. Therefore, there exist infinite solutions in $\mathbb{N}\times\mathbb{N}$. Taking any of these solutions in $\mathbb{N}\times\mathbb{N}$, we can obtain two paths, one from $p$ to $c$ and the other from $p'$ to $c$ of equal lengths. Therefore, $u$ and $u'$ prevent each other from starving.
\end{proof}

Given an $m$-partial SCS, the problem of counting how many of the $m$ robots are in starvation takes polynomial time. This can be done using a {\em starving-prevention test} between every pair of live robots $u$ and $u'$ in the system. 

\begin{quote}
\textbf{Starving-prevention test:}
\emph{Given two robots $u$ and $u'$, are they preventing each other from starving?}
\end{quote}

\begin{lemma}
Given an $m$-partial SCS, the \emph{starving-prevention test} between two robots $u$ and $u'$ in the system takes $O(nT_{gcd}(n))$ time, where $T_{gcd}(n)$ is the time\footnote{$T_{gcd}(n)=O(\log n(\log\log n)^2\log\log\log n)$ according to \cite{gcd}.} to compute the greatest common divisor between any two numbers less than or equal to $n$.
\end{lemma}
\begin{proof}
The first thing to do is a preprocessing in order to find the rings of the system and their lengths. During this process we can also obtain the ties and their lengths. This preprocessing step takes $O(|E|)$ time where $E$ is the set of edges of the communication graph. Since this graph is planar,the preprocessing step takes $O(n)$ time. After that, we proceed as follows in order to perform the prevention test between every pair of robots in the system. Let $u$ and $u'$ be two robots in the system. If $u$ and $u'$ are in the same ring, the test can be done in $O(n)$ time by checking ties of the ring (Lemma~\ref{lem:starve_prevention_ring}). If they are in different rings, $r$ and $r'$ respectively, then there are two options. (i) If $r$ and $r'$ have no common crossing points then they do not prevent each other from starving. (ii) Otherwise, for every crossing point $c$ between $r$ and $r'$ check if $u$ and $u'$ meet each other at $c$ using Theorem~\ref{thm:NP_check}. This can be done in $O(T_{gcd}(n))$ time, where $T_{gcd}(n)$ is the time to compute $gcd(l,l')$ and $2l\pi$ and $2l'\pi$ are the lengths of rings $r$ and $r'$, respectively. Thus, the prevention test takes $O(nT_{gcd}(n))$ time.
\end{proof}

\begin{corollary}
Given an $m$-partial SCS, the total time to count the number of starving robots is $O(m^2nT_{gcd}(n))$.
\end{corollary}

As a consequence of the above results, 
we arrive at the following result:

\begin{corollary}
	The decision problems: determining if the $k$-resilience of a SCS is smaller than a given value $s$ and determining if the starvation number of a SCS is greater than a given value $s$ are both NP-complete, even if the communication graph is a caterpillar tree.
\end{corollary}

\section{Computing $k$-resilience}\label{sec:computing}

The problem of computing the $k$-resilience of a SCS is NP-hard but one may still want to compute the $k$-resilience for a given (presumably small) SCS.
Let $S=\{u_1,\dots,u_n\}$ be a set of $n$ robots in the initial state
 of the system.
One approach is to select a set $S_k\subset S$ of $k$ robots and
\begin{itemize}
\item[(i)] for each robot $u_i$ in $S_k$, find the set of robots $Q_i\subset S$ preventing $u_i$ from starving,
\item[(ii)] find $\displaystyle Q=\bigcup_{u_i\in S_k} Q_i$.
\end{itemize}

Then, the $k$-resilience will be the cardinality of the smallest set $Q$ satisfying $Q\cap S_k=\emptyset$ computed over every possible set $S_k$.
The sets $Q_i$ can be computed using the \emph{starving-prevention test} between every pair of robots in $S_k\times S$. Therefore, the total time to obtain $Q$ is $O(kn^2T_{gcd}(n))$. 
The time for computing the $k$-resilience is $O(kn^{k+2}T_{gcd}(n))$.
We show that the running time can be improved by a factor of $O(nT_{gcd}(n))$.

\subsection{Meeting Graph}
We define a {\em meeting graph} $G_m=(V_m,E_m)$ using robots as vertices and an edge between nodes if the corresponding robots prevent each other from starving.
This graph is useful in computing the $k$-resilience.
We show that it can be computed in $O(n^2)$ time.

{\em Same ring}. For each robot $u$ in a ring $r$, find robots in $r$ at positions $p+l_1,p+l_2,\dots$
where $p$ is the position of $u$ and $\{l_1, l_2 \dots\}$ is the set of the lengths of the ties in $r$. Add edges to $E_m$ between $u$ and every found robot.
The running time of this step is $O(n)$ for one robot $u$ and $O(n^2)$ in total.

{\em Distinct rings}. 
For each crossing point $c$ between two different rings $r$ and $r'$ do the following.
Compute $g=gcd(l,l')$, where $2l\pi$ and $2l'\pi$ are the lengths of $r$ and $r'$ respectively.
For each robot $u$ in $r$ at distance $d$ from $c$, find all robots in $r'$ at distance $d+i\cdot g, i=0,1,2,\dots$.
Add edges between $u$ and these robots to $E_m$.
This can be done in $O(n+I)=O(n)$ time per crossing point where $I$ is the number of edges found for crossing point $c$. The total time is $O(n^2)$.

\subsection{Faster Algorithm}
We show that $k$ robots can be selected in a way that we spend $O(n)$ time for each robot.
Let $A$ be the list of available robots. In the beginning $A=\{u_1,\dots,u_n\}$ contains all $n$ robots. In general, if $A$ is empty, we cannot select a new robot and we check another set of $k$ robots. 
When a robot $u_i$ is selected from $A$, we remove its neighbors in $G_m$ from $A$.
If $u_i$ is not the last ($k$th selected) robot then $A$ contains candidate robots to select next.
If $u_i$ is the last selected robot then $|A|$ is the number of remaining robots.
Let $A_{max}$ be the largest set $A$ of remaining robots over all selections of $k$ robots.
Then $n-k-|A_{max}|$ is the $k$-resilience of the system. This implies the following result.

\begin{theorem}\label{faster}
The $k$-resilience of a SCS can be computed in $O(kn^{k+1})$ time.
\end{theorem}

\subsection{Computing 1-resilience}

\begin{theorem}
The $1$-resilience of a SCS can be computed in $O(nT_{gcd}(n))=\tilde{O}(n)$ time\footnote{$\tilde{O}$ notation hides polylogarithmic factors.}.
\end{theorem}

\begin{proof}
	Denote by $\rho(u)$ the number of robots that prevent robot $u$ from starving. By definition, the value of $1$-resilience is $\min\left\{\rho(u)|u\in S\right\}$ where $S$ is the set of the $n$ robots in the system. 
	
	In a preprocessing step, find the rings and their lengths, also, obtain the ties and their lengths. This preprocessing takes $O(n)$ time. For each ring $r$, pick a single robot $u$ from it. Let $r_1,r_2,\dots$ be the rings crossing $r$. For each ring $r_i$ crossing $r$, compute a list $D_i$ of distances from $u$ to every crossing point between $r$ and $r_i$. Also, in $T_{gcd}(n)$ time, compute $g_i=gcd(l,l_i)$ where $2l\pi$ and $2l_i\pi$ are the lengths of $r$ and $r_i$, respectively. For each distance $d\in D_i$, we compute $d_r$, the remainder of dividing $d$ by $g_i$. Let $n_i$ be the number of distinct $d_r$ for a ring $r_i$. Compute $\rho(u)=t(r)+\sum_i n_i\cdot l_i/g_i$ where $t(r)$ is the distinct lengths of ties in $r$. This is the number of robots preventing $u$ from starvation (Lemma~\ref{lem:starve_prevention_ring} and Theorem~\ref{thm:NP_check}). Let $v$ be another robot laid $2b\pi$ behind $u$ in $r$. Notice that if a robot $u'$ in a ring $r'$ is preventing $u$ from starving, then the robot $v'$ laid $2b\pi$ behind $u'$ in $r'$ prevents $v$ from starving. From this observation we have that for every two robots $u$ and $v$ in the same ring, $\rho(u)=\rho(v)$. Let $\varrho(r)$ be the value $\rho(u)$ of an arbitrary robot in the ring $r$. Computing $\varrho(r)$ takes $O(e_r\cdot T_{gcd}(n))$ time where $e_r$ is the number of crossing points traversed by $r$. Find the smallest $\rho(r)$ by computing it for all rings $r$ in the system. The total running time is $O(nT_{gcd}(n))=\tilde{O}(n)$ because every crossing point is analyzed two times, one per each traversing ring (the same ring twice in ties) and the number of crossing points is $O(n)$. 
    \end{proof}

	
	

\section{Computing $k$-resilience for trees}\label{sec:trees}
Trees constitute an important family of graphs for computing $k$-resilience. For instance, spanning trees can be used to enforce synchronization on non-planar communication graphs. Furthermore, the communication graph of a tree has only ring. 

\begin{lemma}
	In a SCS whose communication graph is a tree the 1-resilience can be computed in $O(n)$ time.
\end{lemma}
\begin{proof} 
	{\em 1-resilience}. 
	Recall that there is only one ring if $G$ is a tree.
	We can pick any robot $u$ for starvation. 
	The robots preventing $u$ from starvation can be found from the ties.
	Compute all tie lengths $l_1<..<l_t$ using the computation of rings. 
	Then, the 1-resilience of $G$ is $t$.
	\end{proof}

\begin{lemma} \label{2res}
	In a SCS whose communication graph is a tree the 2-resilience can be computed in $O(t^2)$ time where $t$ is the number of distinct tie lengths.
\end{lemma}

\begin{proof} 
	Compute tie lengths $L=\{l_1,\dots,l_t\}$ using the computation of rings. Compute the mode $m$ of multiset $S_+\cup S_-$ where $S_+$ and $S_-$ are the multisets:
	\[
	\{l_i+l_j<n, l_i+l_j\notin L\} \text{ and } 
	\{l_i-l_j>0, l_i-l_j\notin L\}\text{, respectively}. 
	\]
	Then the 2-resilience is $2t-f$ where $f$ is the frequency of $m$ in $S_+\cup S_-$.
	
	We show that the algorithm is correct.
    For convenience we use a circular numbering of the $n$ robots in the system $0,1,\dots,n-1$ such that two robots $i$ and $j$ ($j>i$) prevent each other from starving if and only if $j-i\in L$.
	Consider an optimal solution. W.l.o.g. assume that 2 starving robots are at positions 0 and $d$. 
	There are $t$ robots preventing robot 0 from starving. 
	Their positions are $S_0=L$. 
	There are $t$ robots preventing robot $d$ from starving. 
	Their positions are $S_d=\{d+l_1,\dots,d+l_t\}$ using modulo $n$. 
	The number of robots preventing robot 0 or $d$ from starving is $2t-|S_0\cap S_d|$ which is 2-resilience. 
	
	The number of robots preventing both robots 0 and $d$ is $|S_0\cap S_d|$. 
	Consider a robot from $S_0\cap S_d$ at position $x$.
	Clearly $x\in L$.
	If $x<d$ then $d=x+l_j$ and $d\in S_+$. 
	If $x>d$ then $d=x-l_j$ and $d\in S_-$. 
	Therefore $|S_0\cap S_d|$ is equal to the frequency of $d$ in $S_+\cup S_-$ and the 2-resilience is equal to $2t-f$.
	
	The running time is $O(t^2+n)$ time. By Lemma \ref{lowerb}, this is simply $O(t^2)$.
	\end{proof}

\begin{lemma} \label{lowerb}
In a SCS whose communication graph is a tree, if there are $t$ distinct tie lengths then $\frac {\sqrt{\pi n}}2-1\le t\le n-1$.
\end{lemma}

\begin{proof} 
	To show the upper bound, we observe that each tie has length $2\pi a,a\in \{1,2,\dots,n-1\}$.
	We use a packing argument for the lower bound.
	Let $s_n$ be side length of the smallest square containg $n$ unit circles.
	Since every circle has area $\pi$, we have $s_n>\sqrt{\pi n}$.
	There are better lower bounds for small $n$ (up to 100) \cite{Casado2001,Maranas,Szabo2001}.
	Szabo et.al. \cite{Szabo2001} provide the following bound 
	\[
	m_n\le \frac{1+\sqrt{1+2(n-1)/\sqrt 3}}{n-1},
	\]
	where $m_n$ is the largest value of $\min_{i<j}|p_ip_j|$ for any set of points $p_1,\dots,p_n$ in the unit square.
	Clearly, $s_n=2+2/m_n$.
	
	Compute $x_{min},x_{max},y_{min},y_{max}$ for the centers of the trajectories. 
	Then $|x_{max}-x_{min}|\ge s_n-2$ or $|y_{max}-y_{min}|\ge s_n-2$.
	Let $T$ be the communication graph of the system.
    Then the diameter $\delta$ of $T$ is at least $(s_n-2)/2+1=s_n/2$. 
	Let $C_1,..,C_\delta$ be the sequence of trajectories in the system forming the diameter of $T$. It has $\delta-1$ edges. Each edge determines two ties. There are $2(\delta-1)$ ties. Clearly at most two ties have the same length.
	So, we have at least $\delta-1$ distinct tie lengths.
	The lower bound follows since $\delta\ge s_n/2>\sqrt{\pi n}/2$.
	\end{proof}

For any $n$, there is an instance of a tree of size $n$ with $\Omega(\sqrt n)$ tie lengths.
If $n=a^2$ then the tree is built of $a$ paths of length $a$ which are connected as shown in Figure \ref{g2}.
The number of tie lengths (not the number of ties!) is $a+2(a-1)=3a-2$.
In general, we take $a=\lfloor\sqrt n\rfloor$ and add $n-a^2$ trajectories in the middle of the tree.

\begin{figure}[htb]
	\centering
	\includegraphics[scale=0.8]{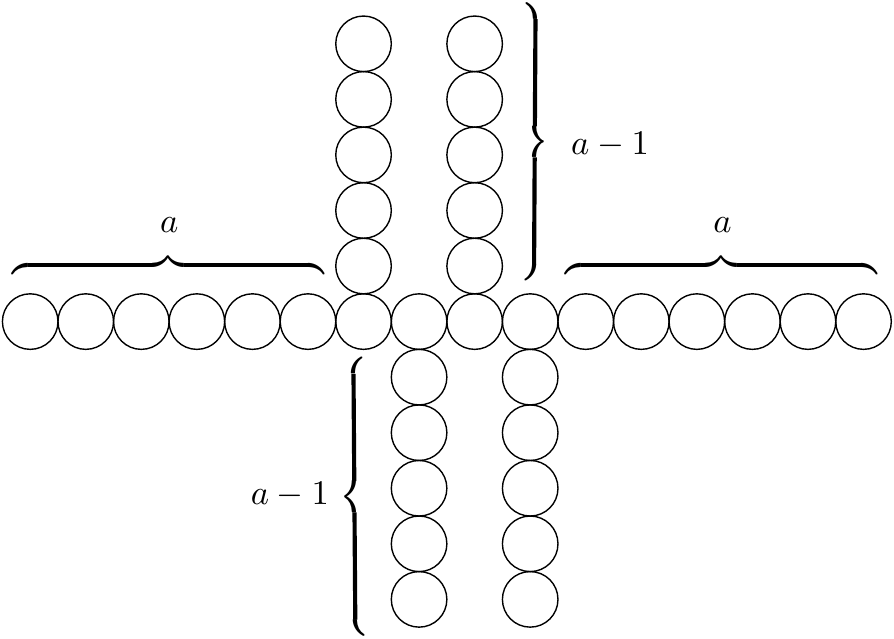}
	\caption{A tree for $n=a^2$. There are $n-1$ ties and they have lengths $1,2,\dots,a,2a,3a,\dots,n-a,a-a+1,\dots,n$.}
	\label{g2}
\end{figure}

Notice that if the number $t$ of distinct tie lengths in the tree is $O(\sqrt{n})$ then the 2-resilience can be computed in linear time. In the worst case, when $t$ is $\Omega(n)$, the 2-resilience can be computed in $O(n^2)$ time. The last complexity can be improved if the problem \textbf{mode-of-differences} (described in Section~\ref{ch:conclusions}) can be solved in subquadratic time.

\begin{theorem}
In a SCS whose communication graph is a tree, the $k$-resilience, $k\ge 3$, can be computed in $O(n^{k-2}t\cdot\min(n,kt))=O(tn^{k-1})$ time, where $t$ is the number of distinct tie lengths.
\end{theorem}

\begin{proof}
	Compute tie lengths $L=\{l_1,\dots,l_t\}$ of the ring.
    For convenience we use a circular numbering of the $n$ robots in the system $0,1,\dots,n-1$ such that two robots $i$ and $j$ ($j>i$) prevent each other from starving if and only if $j-i\in L$.
	Place one robot at position 0 (this can be done without loss of generality).
	Choose $k-2$ robots making a set $S'$ of $k-1$ robots. 
	This can be done as in the algorithm from Theorem~\ref{faster} (using the list $A$ of available robots).
	The running time for testing $k-1$ robots is $O((k-1)t)=O(kt)$.
	
	Now the task is to find a robot in $A$ with the minimum number of robots in $A$ preventing it from starving.
	Let $F=\{f_1,f_2,\dots,f_z\}$ be the set of robots preventing at least one robot in $S'$ from starving.
	Clearly, $F=\{0,1,2,\dots,n-1\}-(A\cup S')$. 
	Compute the mode $m$ of multiset $S_+\cup S_-$ where 
	\[
	S_+=\{f_i+l_j, f_i+l_j\in A\} \text{ and } 
	S_-=\{f_i-l_j, f_i-l_j\in A\}. 
	\]
	Compute $\rho(S')=|F|+t-f$ where $f$ is the frequency of $m$ in $S_+\cup S_-$. Compute the $k$-resilience as minimum of $\rho(S')$ over all possible sets $S'$. 
	
	Clearly, $|F|\le \min(n,kt)$ and the algorithm has the running time $O(n^{k-2}t\cdot\min(n,kt))$.
	The algorithm is correct by an argument similar to the proof of
	Lemma \ref{2res}.
\end{proof}

\section{Concluding remarks}\label{ch:conclusions}


In this work we have studied a combinatorial optimization problem related to the robustness of synchronized systems composed of robots that cooperate to cover an area with constrained communication range. 
We stated the concept of \emph{starvation} as a phenomenon that can appear when a set of robots leaves a synchronized system. This phenomenon is characterized by the permanent loss of communication of one or more surviving agents when a number of robots leave a synchronized system. Also, we present the \emph{starvation state} of a system as an extreme case of communication breakdown, where all the surviving robots in the system are permanently isolated. Then we addressed the main topic of this work, the \emph{$k$-resilience} of a system, defined as the cardinality
of a smallest set of robots whose failure suffices
to cause that at least $k$ surviving robots become incommunicado.
We prove that the problem is NP-complete when $k$ is part of the input and propose efficient algorithms for small values of $k$. 

A possible research line is to improve the time complexity  for constant values of $k$.
A possibility is to follow our approach and solve in sub-quadratic time the following basic questions that we state here as new open problems in algorithm design, 
related to 2- and $k$-resilience, respectively.

{\bf The mode-of-differences problem:} \emph{Let $0<l_1<l_2<\dots<l_t<n$ be $t$ integer numbers, $t\in\Omega(n)$. For each $0<i<n$, let $R_i$ is the number of times that $i$ is the difference of two of the given numbers. Compute $\max R_i$.}

{\bf The mode-of-differences-of-two-sets problem:} \emph{Let $A$ and $B$ be two subsets of $\{1,2,\dots,n-1\}$ whose cardinalities are in $\Omega(n)$.
For each $0<i<n$, let $R_i$ be the number of times $i$ appears in the multiset $A-B$. Compute $\max R_i$.}







\bibliography{bibliography}
\bibliographystyle{plain}

\end{document}